\documentclass[12pt]{article}

\usepackage[english]{babel}

\usepackage[letterpaper,top=2cm,bottom=2cm,left=3cm,right=3cm,marginparwidth=1.75cm]{geometry}

\usepackage{hyperref}
\hypersetup{colorlinks=True,allcolors=black,breaklinks=true}
\usepackage{breakcites}
\usepackage[plain]{algorithm}
\usepackage{algorithmic}
\usepackage{balance}
\usepackage{tcolorbox, bm, booktabs, adjustbox, multirow}
\usepackage{amsmath, amsthm, amsfonts, bbm}
\newtheorem{theorem}{Theorem}[section]

\newtheorem{lemma}[theorem]{Lemma}
\newtheorem{claim}[theorem]{Claim}

\theoremstyle{definition}
\newtheorem{definition}{Definition}[section]

\newcommand{\INs}{\mathcal{X}}
\newcommand{\OUTs}{\mathcal{Y}}
\newcommand{\R}{\mathbb{R}}

\newcommand{\w}{\mathbf{w}}
\newcommand{\Expect}{\mathbb{E}}
\newcommand{\Prob}{\mathbb{P}}
\newcommand{\obs}{\mathbf{x}}
\newcommand{\Dist}{\mathcal{D}}
\newcommand{\Hypo}{\mathcal{H}}
\newcommand{\LblTrn}{\mathbf{S}_\ell}
\newcommand{\UnLblTrn}{\mathbf{X}_u}

\renewcommand{\S}{\mathbf{S}}

\newdimen\AAdi%
\newbox\AAbo%
\def\AAk#1#2{\setbox\AAbo=\hbox{#2}\AAdi=\wd\AAbo\kern#1\AAdi{}}%
\def\AAr#1#2#3{\setbox\AAbo=\hbox{#2}\AAdi=\ht\AAbo\raise#1\AAdi\hbox{#3}}%

\makeatletter
\newcommand*{\rom}[1]{\expandafter\@slowromancap\romannumeral #1@}
\newcommand{\argmin}{\mathop{\mathrm{argmin}}\limits}
\newcommand*\rel@kern[1]{\kern#1\dimexpr\macc@kerna}
\newcommand*\widebar[1]{%
  \begingroup
  \def\mathaccent##1##2{%
    \rel@kern{0.8}%
    \overline{\rel@kern{-0.8}\macc@nucleus\rel@kern{0.2}}%
    \rel@kern{-0.2}%
  }%
  \macc@depth\@ne
  \let\math@bgroup\@empty \let\math@egroup\macc@set@skewchar
  \mathsurround\z@ \frozen@everymath{\mathgroup\macc@group\relax}%
  \macc@set@skewchar\relax
  \let\mathaccentV\macc@nested@a
  \macc@nested@a\relax111{#1}%
  \endgroup
}
\makeatother

\title{Self-Training of Halfspaces with Generalization Guarantees under Massart Mislabeling Noise Model}
\author{Lies Hadjadj$^\dagger$, Massih-Reza Amini$^\dagger$, Sana Louhichi$^\ddagger$, Alexis Deschamps$^\star$\\ 
\{Firstname.Lastname\}@univ-grenoble-alpes.fr\\
$^\dagger$ Computer Science Laboratory (LIG)\\
$^\ddagger$ Department of Statistics (LJK)\\
$^\star$ Departement of Physics (SIMaP) \\
University of Grenoble Alpes, France}
\date{}
\begin{document}
\maketitle

\begin{abstract}
We investigate the generalization properties of a self-training algorithm with halfspaces. The approach learns a list of halfspaces iteratively from labeled and unlabeled training data, in which each iteration consists of two steps: exploration and pruning. In the exploration phase, the halfspace is found sequentially by maximizing the unsigned-margin among  unlabeled examples and then assigning pseudo-labels to those that have a distance higher than the current threshold. The pseudo-labeled examples are then added to the training set, and a new classifier is learned. This process is repeated until no more unlabeled examples remain for pseudo-labeling. In the pruning phase, pseudo-labeled samples that have a distance to the last halfspace greater than the associated  unsigned-margin are then discarded. We prove that the misclassification error of the resulting sequence of classifiers is bounded and show that the resulting semi-supervised approach never degrades performance compared to the  classifier learned using only the initial labeled training set. Experiments carried out on a variety of benchmarks demonstrate the efficiency of the proposed approach compared to state-of-the-art methods.
\end{abstract}

\section{Introduction}

In recent years, several attempts have been made to establish a theoretical foundation for semi-supervised learning. These studies are mainly interested in the generalization ability of semi-supervised learning techniques \cite{Rigollet07,MaximovAH18} and the utility of unlabeled data in the training process \cite{CastelliC95,Singh09,ICML-2011-LiZ}. The majority of these works are based on the concept called \textit{compatibility} in \cite{Bal06}, and try to exploit the connection between the marginal data distribution and the target function to be learned. The common conclusion of these studies is that unlabeled data will only be useful for training if such a relationship exists. 

The three key types of relations considered in the literature are cluster assumption, manifold assumption, and low-density separation \cite{Zhu05,Chap06}. The cluster assumption states that data contains homogeneous labeled clusters, and unlabeled training examples allow to recognize these clusters. In this case, the marginal distribution is viewed as a mixture of class conditional distributions, and semi-supervised learning has been shown to be superior to supervised learning in terms of achieving smaller finite-sample error bounds in some general cases, and in some others, it provides a faster rate of error convergence \cite{CastelliC95,Rigollet07,MaximovAH18, Singh09}. In this line, \cite{Ben-DavidLP08} showed that the access to the marginal distribution over unlabeled training data would not provide sample size guarantees better than those obtained by supervised learning unless one assumes very strong assumptions about the conditional distribution over the class labels.  Manifold assumption stipulates that the target function is in a low-dimensional manifold. \cite{Niyogi13a} establishes a context through which such algorithms can be analyzed and potentially justified; the main result of this study is that unlabeled data may help the learning task in certain cases by defining the manifold. Finally, low-density separation states that the decision boundary lies in low-density regions. A principal way, in this case, is to employ a margin maximization strategy which results in pushing away the decision boundary from the unlabeled data \cite[ch. 6]{Chap06}. Semi-supervised approaches based on this paradigm mainly assign pseudo-labels to high-confident unlabeled training examples with respect to the predictions and include these pseudo-labeled samples in the learning process. However, \cite{chaw11} investigated empirically the problem of label noise bias introduced during the pseudo labeling process in this case and showed that the use of unlabeled examples could have a minimal gain or even degrade performance, depending on the generalization ability of the initial classifier trained over the labeled training data.   

In this paper, we study the generalization ability of a self-training algorithm with halfspaces that operates in two steps. In the first step, halfspaces are found iteratively over the set of labeled and unlabeled training data by maximizing the unsigned-margin of unlabeled examples and then assigning pseudo-labels to those with a distance greater than a found threshold.  The pseudo-labeled unlabeled examples are then added to the training set, and a new classifier is learned. This process is repeated until there are no more unlabeled examples to pseudo-label. In the second step, pseudo-labeled examples with an unsigned-margin greater than the last found threshold are removed from the training set. 

Our contribution is twofold: $(a)$ we present a first generalization bound for self-training with halfspaces in the case where class labels of examples are supposed to be corrupted by a Massart noise model; $(b)$ We show that the use of unlabeled data in the proposed self-training algorithm does not degrade the performance of the first halfspace trained over the labeled training data.

In the remainder of the paper, Section \ref{sec2} presents the definitions and the learning objective. In Section \ref{sec3}, we present in detail the adaptation of the self-training algorithm for halfspaces. Section \ref{sec4} presents a bound over the misclassification error of the classifier outputted by the proposed algorithm and demonstrates that this misclassification error is upper-bounded by the misclassification error of the fully supervised halfspace.  In Section \ref{sec5}, we present experimental results, and we conclude this work in Section \ref{sec6}. 

\section{Framework and Notations}
\label{sec2} 
 We consider binary classification problems where the input space $\INs$ is a subset of $\R^d$, and the output space is $\OUTs=\{-1,+1\}$. We study learning algorithms that operate in hypothesis space $\Hypo_d=\{h_\w: \INs \rightarrow \OUTs\}$ of centered halfspaces, where each $h_\w\in \Hypo_d$ is a Boolean function of the form $h_\w(\obs) = \text{sign}(\langle \w,\obs \rangle)$, with $\w \in~\R^d$ such that $\|\w\|_2 \leq 1$. 

 Our analysis succeeds the recent theoretical advances in robust supervised learning of polynomial algorithms for training halfspaces under large margin assumption \cite{Diak19, Mon20, Diak20}, where the label distribution has been corrupted with the Massart noise model \cite{mas06}. These studies derive a PAC bound for generalization error for supervised classifiers that depends on the corruption rate of the labeled training set and shed light on a new perspective for analyzing the self-training algorithm. Similarly, in our analysis, we suppose that self-training can be seen as learning with an imperfect expert. Whereat at each iteration, labels of the pseudo-labeled set have been corrupted with a Massart noise \cite{mas06} oracle defined as:
 
 \begin{definition}[\cite{mas06} noise oracle] \label{def}
Let $\mathcal{C}=\{f:\INs\rightarrow\OUTs\}$ be a class of Boolean functions over $\INs \subseteq \R^d$, with $f$  an unknown target function in $\mathcal{C}$, and $0 \leq \boldsymbol{\eta} < 1/2$. Let $\eta$ be an unknown parameter function such that $\Expect_{\obs \sim \Dist_\obs}[\eta(\obs)]  \leq \boldsymbol{\eta}$, with $\Dist_\obs$  any marginal distribution over $\INs$. The corruption oracle $\mathcal{O}(f, \Dist_\obs, \eta)$ works as follow: each time $\mathcal{O}(f, \Dist_\obs, \eta)$ is invoked, it returns a pair $(\obs, y)$ where $\obs$ is generated i.i.d. from $\Dist_\obs$; $y = -f(\obs)$ with probability $\eta(\obs)$ and $y=f(\obs)$ with probability $1-\eta(\obs)$.
\end{definition}

Let $\Dist$ denote the joint distribution over $\INs \times \OUTs$ generated by the above oracle with an unknown parameter function $\eta$ such that $\Expect_{\obs \sim \Dist_\obs}[\eta(\obs)]  \leq \boldsymbol{\eta}$. We suppose that the training set is composed of $\ell$ labeled samples $\LblTrn~=~(\obs_i,y_i)_{1\leq i\leq \ell}\in(\INs \times \OUTs)^\ell$ and $u$ unlabeled samples $\UnLblTrn~=~(\obs_i)_{\ell+1\leq i \leq \ell+u}\in~\INs^u$, where $\ell <\!\!< u$. Furthermore, we suppose that each pair $(\obs,y)\in\INs\times\OUTs$ is i.i.d. with respect to the probability distribution $\Dist$, we denote by $\Dist_\obs$ the marginal of $\Dist$ on $\obs$, and $\Dist_y(\obs)$ the distribution of $y$ conditional on $\obs$. Finally, for any  integer $d$, let  $[d] = \{0,...,d\}$.

\subsection{Learning objective}

Given $\LblTrn$ and $\UnLblTrn$, our goal is to find a learning algorithm that outputs a hypothesis $h_\w \in \Hypo_d$ such that with high probability, the misclassification error $\Prob_{(\obs, y) \sim \Dist}{[h_\w(\obs) \neq y]}$ is minimized and to show with high probability that the performance of such algorithm is better or equal to any hypothesis in $\mathcal{H}_d$ obtained from $\LblTrn$ only. Here we denote by $\mathbf{\eta}_\w(\obs) = \Prob_{y \sim \Dist_y(\obs)}{[h_\w(\obs) \neq y]}$ the conditional misclassification error of a hypothesis $h_{\w}\in\Hypo_d$ with respect to $\Dist$ and $\w^*$ the normal vector of $h_{\w^*}\in~\Hypo_d$ that achieves the optimal misclassification error; $\pmb{\eta}^* =~\underset{\w, \|\w\|_2 \leq 1}{\min} \; \Prob_{(\obs, y) \sim \Dist}{[h_\w(\obs) \neq y]}$.

By considering the indicator function $\mathbbm{1}_\pi$ defined as $\mathbbm{1}_\pi=1$ if the predicate $\pi$ is true and $0$ otherwise; we prove in the following lemma that the probability of misclassification of halfspaces over examples with an unsigned-margin greater than a threshold $\gamma>0$ is bounded by the same quantity $1>\pmb{\eta}>0$ that upper-bounds the misclassification error of these examples.

\begin{lemma}\label{lem1}
For all $h_\w \in \Hypo_d\,$, if there exist $\pmb{\eta}\in ]0,1[$ and $\gamma > 0$ such that:\\$\Prob_{\obs \sim \Dist_\obs}{[|\langle\w, \obs\rangle| \geq \gamma]}> 0$ and that $\Expect_{\obs \sim \Dist_\obs}{[(\eta_\w(\obs) - \pmb{\eta})\mathbbm{1}_{|\langle\w, \obs\rangle| \geq \gamma}] \leq 0}$, then:\\$\Prob_{(\obs, y) \sim \Dist}{[h_\w(\obs) \neq y \big| |\langle\w, \obs\rangle| \geq \gamma]} \leq \pmb{\eta}$. 
\end{lemma}

\begin{proof}
For all hypotheses $h_\w$ in $\Hypo_d$, we know that the error achieved by $h_\w$ in the region of margin $\gamma$ from $\w$ satisfies $\Expect_{\obs \sim \Dist_\obs}{[(\eta_\w(\obs) - \pmb{\eta})\mathbbm{1}_{|\langle\w, \obs\rangle| \geq \gamma}] \leq 0}$; by rewriting the expectation, we obtain the following $    \Expect_{\obs \sim \Dist_\obs}{[\eta_\w(\obs)\mathbbm{1}_{|\langle\w, \obs\rangle| \geq \gamma}] - \pmb{\eta}\Prob_{\obs \sim \Dist_\obs}{[|\langle\w, \obs\rangle| \geq \gamma]} \leq 0}$. We have then  $    \frac{\Expect_{\obs \sim \Dist_\obs}{[\eta_\w(\obs)\mathbbm{1}_{|\langle\w, \obs\rangle| \geq \gamma}]}} {\Prob_{\obs \sim \Dist_\obs}{[|\langle\w, \obs\rangle| \geq \gamma]}} \leq \pmb{\eta}$ and the result follows from the equality:
\begin{equation*}
\Prob_{(\obs, y) \sim \Dist}{[h_\w(\obs) \neq y \big| |\langle\w, \obs\rangle| \geq \gamma]}=\frac{\Expect_{\obs \sim \Dist_\obs}{[\eta_\w(\obs)\mathbbm{1}_{|\langle\w, \obs\rangle| \geq \gamma}]}} {\Prob_{\obs \sim \Dist_\obs}{[|\langle\w, \obs\rangle| \geq \gamma]}}.
\end{equation*}
\end{proof}

Suppose that there exists a pair $(\widetilde{\w}, \widetilde \gamma)$ minimizing:
\begin{equation}\label{op}
(\widetilde\w, \widetilde \gamma)\in \underset{\w\in \R^d,\gamma\geq 0}{\argmin}\frac{\Expect_{\obs \sim \Dist_\obs}{[\eta_\w(\obs)\mathbbm{1}_{|\langle\w, \obs\rangle| \geq \gamma}]}} {\Prob_{\obs \sim \Dist_\obs}{[|\langle\w, \obs\rangle| \geq \gamma]}}.
\end{equation} By defining $\widetilde \eta$ as:
\begin{equation*}
{\widetilde \eta} 
 = \underset{\w\in \R^d,\gamma\geq 0}{\inf}\frac{\Expect_{\obs \sim \Dist_\obs}{[\eta_\w(\obs)\mathbbm{1}_{|\langle\w, \obs\rangle| \geq \gamma}]}} {\Prob_{\obs \sim \Dist_\obs}{[|\langle\w, \obs\rangle| \geq \gamma]}}.
\end{equation*}
The following inequality holds:
\[
{\widetilde \eta} \leq \underset{\w\in \R^d}{\inf} \frac{\Expect_{\obs \sim \Dist_\obs}{[\eta_\w(\obs)\mathbbm{1}_{|\langle\w, \obs\rangle| \geq 0}]}} {\Prob_{\obs \sim \Dist_\obs}{[|\langle\w, \obs\rangle| \geq 0]}}= \pmb{\eta}^*.
\]
This inequality paves the way for the following claim, which is central to the self-training strategy described in the next section.
\begin{claim} \label{claim_main} Suppose that there exists  a pair 
$(\widetilde\w, \widetilde\gamma)$ satisfying the minimization problem (\ref{op}) with 
$\;\Prob_{\obs \sim \Dist_\obs}{[|\langle\widetilde\w, \obs\rangle| \geq \widetilde\gamma]} > 0\;$, then  $\;\Prob_{(\obs, y) \sim \Dist}{[h_{\widetilde\w}(\obs) \neq y \big| |\langle \widetilde\w, \obs\rangle| \geq \widetilde \gamma]} \leq \pmb{\eta}^*$.
\end{claim}
\begin{proof}
The requirements of Lemma \ref{lem1} are satisfied with
$(\w,\gamma)=(\widetilde\w, \widetilde \gamma)$ and $\eta=\widetilde \eta$. This claim is then proved using the conclusion of 
Lemma \ref{lem1} together with the fact that $\widetilde \eta\leq \boldsymbol{\eta}^*$.
\end{proof}

The claim above demonstrates that for examples generated by the probability distribution $\Dist$, there exists a region in $\INs$ on either side of a margin $\widetilde \gamma$ to the decision boundary defined by $\widetilde \w$ solution of (Eq. \ref{op}); where the probability of misclassification error of the corresponding halfspace in this region is upper-bounded by the optimal misclassification error $\boldsymbol{\eta}^*$. This result is consistent with semi-supervised learning studies that consider the margin as an indicator of confidence and search the decision boundary on low-density regions \cite{Joa99,Gra05,NIPS2008_dc6a7071}. 

\subsection{Problem resolution}
We use a block coordinate minimization method for solving the optimization problem~\eqref{op}. This strategy consists in first finding a halfspace with parameters $\widetilde \w$ that minimizes Eq.~\eqref{op} with a threshold $\gamma=0$, and then by fixing $\widetilde \w$, finds the threshold $\widetilde \gamma$ for which Eq.~\eqref{op} is minimum. We resolve this problem using the following claim, which links the misclassification error $\eta_\w$ and the perceptron loss $\ell_p(y,h_\w(\obs)): \OUTs\times~\OUTs\rightarrow
~\mathbb{R}_+;  \ell_p(y,h_\w(\obs))=-y\langle \w, \obs \rangle\mathbbm{1}_{y\langle \w, \obs \rangle \leq 0}$.
\begin{claim}
\label{claim_err_loss}
For a given weight vector $\w$, we have:
\begin{equation}
\label{eq:Claim2.3}
\Expect_{\obs \sim \Dist_\obs}[|\langle \w, \obs \rangle|\eta_{\w}(\obs)]=\Expect_{(\obs,y) \sim \Dist}[ \ell_p(y,h_\w(\obs))]
\end{equation}
\end{claim}
\begin{proof}
For a fixed weight vector $\mathbf{w}$, we have that:
$\Expect_{(\obs,y) \sim \Dist}[\ell_p(y,h_\w(\obs))]  = \\ \Expect_{(\obs,y) \sim \Dist}[-y\langle \w, \obs \rangle \mathbbm{1}_{y\langle \w, \obs \rangle \leq 0}]$. As we are considering misclassification errors,\\i.e., $-y\langle \w, \obs \rangle \mathbbm{1}_{y\langle \w, \obs \rangle \leq 0}=\mathbbm{1}_{y\langle \w, \obs \rangle \leq 0} |\langle \w, \obs \rangle|$, it comes that $\Expect_{(\obs,y) \sim \Dist}[\ell_p(y,h_\w(\obs))]  =\\\Expect_{(\obs,y) \sim \Dist}[|\langle \w, \obs \rangle|\Prob_{y\sim\Dist_{y(\obs)}}[-y\langle \w, \obs \rangle > 0] ]$. The result then follows from the definition of the misclassification error, i.e., $\eta_\w(\obs)=\Prob_{y\sim\Dist_{y(\obs)}}[-y\langle \w, \obs \rangle > 0]$.
\end{proof}
This claim shows that the minimization of the generalization error with $\ell_p$ is equivalent to minimizing $\Expect_{\obs \sim \Dist_\obs}[|\langle \w, \obs \rangle|\eta_{\w}(\obs)]$. Hence, the minimization of $\Expect_{\obs \sim \Dist_\obs}[\ell_p(y,h_\w(\obs))]$ cannot result in bounded misclassification error, as the distribution of margins $|\langle \w, \obs \rangle|$ might vary widely between samples in $\INs$. In the following lemma, we show that it is possible to achieve bounded misclassification error under margin condition and $L_2$-norm constraint.  

\begin{lemma} \label{lem11} For a fixed distribution $\Dist$, let $R =~\displaystyle\mathop{\max}_{\obs \sim \Dist_\obs}\|\obs\|_2$ and $\gamma > 0$, let $\widetilde\w$ and $\bar\w$ be defined as follows:
\begingroup
\setlength\abovedisplayskip{5pt}
\begin{align*}
    & \widetilde\w  = \argmin_{\w, ||\w||_2 \leq 1} \Expect_{\obs \sim \Dist_\obs}{[|\langle\w, \obs\rangle| \eta_\w(\obs) \big| |\langle\w, \obs\rangle| \geq \gamma]}\\
    & \widebar\w = \argmin_{\w, ||\w||_2 \leq 1} \Expect_{\obs \sim \Dist_\obs}{[\eta_\w(\obs) \big| |\langle\w, \obs\rangle| \geq \gamma]}. 
\end{align*}
\endgroup
We then have:\\
\begin{equation*}
\frac{\gamma}{R} \Expect_{\obs \sim \Dist_\obs}{[\eta_{\widetilde\w}(\obs) \big| |\langle\widetilde\w, \obs\rangle| \geq \gamma]} \leq \Expect_{\obs \sim \Dist_\obs}{[\eta_{\widebar\w}(\obs) \big| |\langle\widebar\w, \obs\rangle| \geq \gamma]} \leq\Expect_{\obs \sim \Dist_\obs}{[\eta_{\widetilde\w}(\obs) \big| |\langle\widetilde\w, \obs\rangle| \geq \gamma]}.
\end{equation*}
\end{lemma}

\begin{proof} From the condition $|\langle\widetilde\w, \obs\rangle| \geq \gamma$ in the expectation, we have:
\begin{align*}
\gamma \Expect_{\obs \sim \Dist_\obs}{[\eta_{\widetilde\w}(\obs) \big| |\langle\widetilde\w, \obs\rangle| \geq \gamma]} \leq \Expect_{\obs \sim \Dist_\obs}{[|\langle\widetilde\w, \obs\rangle| \eta_{\widetilde\w}(\obs) \big| |\langle\widetilde\w, \obs\rangle| \geq \gamma]}
\end{align*}
Applying the definition of $\widetilde\w$ to the right-hand side of the above inequality gives:

\begin{align*}
\gamma \Expect_{\obs \sim \Dist_\obs}{[\eta_{\widetilde\w}(\obs) \big| |\langle\widetilde\w, \obs\rangle| \geq \gamma]} \leq  \Expect_{\obs \sim \Dist_\obs}{[|\langle\widebar\w, \obs\rangle| \eta_{\widebar\w}(\obs) \big| |\langle\widebar\w, \obs\rangle| \geq \gamma]}
\end{align*}
Using the Cauchy–Schwarz inequality and the definition of $R$, we get:

\begin{align*}
\gamma \Expect_{\obs \sim \Dist_\obs}{[\eta_{\widetilde\w}(\obs) \big| |\langle\widetilde\w, \obs\rangle| \geq \gamma]} \leq R~\Expect_{\obs \sim \Dist_\obs}{[\eta_{\widebar\w}(\obs) \big| |\langle\widebar\w, \obs\rangle| \geq \gamma]}
\end{align*}

Then from the definition of $\widebar\w$, we know:
\begin{align*}
R~\Expect_{\obs \sim \Dist_\obs}{[\eta_{\widebar\w}(\obs) \big| |\langle\widebar\w, \obs\rangle| \geq \gamma]} \leq R~\Expect_{\obs \sim \Dist_\obs}{[\eta_{\widetilde\w}(\obs) \big| |\langle\widetilde\w, \obs\rangle| \geq \gamma]}
\end{align*}
Dividing the two inequalities above by $R$ gives the result.
\end{proof}

Lemma~\ref{lem11} guarantees that the approximation of the perceptron loss to the misclassification error is more accurate for examples that have a comparable distance to the halfspace. This result paves the way to our implementation of the self-learning algorithm.


\section{Self-Training with Halfspaces}
\label{sec3}

Given $\mathbf{S}_\ell$ and $\mathbf{X}_u$ drawn i.i.d. from a distribution $\mathcal{D}$ corrupted with $\mathcal{O}(f, \mathcal{D}_\mathbf{x}, \eta^{(0)})$. Algorithm~1 learns iteratively a list of halfspaces $L_m=[(\mathbf{w}^{(1)},\gamma^{(1)}),...,(\mathbf{w}^{(m)},\gamma^{(m)})]$ with each round consisting of \textit{exploration} and \textit{pruning} steps. 

\begin{algorithm}[t!]
\begin{center}
\tcbset{width=\columnwidth,before=,after=\hfill, colframe=black, colback=white, fonttitle=\bfseries, coltitle=black, colbacktitle=gray!20, boxrule=0.2mm,arc=1mm, left = 2pt}
\begin{tcolorbox}[title=Algorithm 1. Self-Training with Halfspaces]
\begin{algorithmic}
\STATE {\bfseries Input :} $\mathbf{S}_\ell = (\mathbf{x}_i,y_i)_{1\leq i\leq l}$, $\mathbf{X}_u = (\mathbf{x}_i)_{l+1\leq i\leq n}$, p: number of threshold tests set to $5$.\\
\STATE Set $k \leftarrow 0$, $\mathbf{S}^{(k)} = \mathbf{S}_\ell$, $\mathbf{U}^{(k)} = \mathbf{X}_u$, $w = \frac{|\mathbf{S}^{(k)}|}{p}$, $L=[]$.\\
\WHILE{$|\mathbf{S}^{(k)}| \geq \ell$}
\STATE Let $\mathcal{\hat R}_{\mathbf{S}^{(k)}}(\mathbf{w}) = \frac{1}{|\mathbf{S}^{(k)}|} \sum_{(\mathbf{x},y) \in \mathbf{S}^{(k)}} {\left[\ell_p(y,h_\w(\obs))\right]}$\\
Run projected SGD on $\mathcal{\hat R}_{{\mathbf{S}}^{(k)}}(\mathbf{w})$ to obtain $\mathbf{w}^{(k)}$ such that $\|\mathbf{w}^{(k)}\|_2 \leq 1$.\\
\STATE Order $\mathbf{S}^{(k)}$ by decreasing order of margin from $\mathbf{w}^{(k)}$.\\
\STATE Set a window of indices $I = [w, 2w, ..., pw]$, \\
\STATE find $t = \argmin_{i \in I} \frac{1}{|\mathbf{S}_{\geq i}^{(k)}|} \sum_{(\mathbf{x},y) \in \mathbf{S}_{\geq i}^{(k)}} \mathbbm{1}_{h_{\mathbf{w}^{(k)}}(\mathbf{x}) \neq y}$.\\
\STATE Set $\gamma^{(k)}$ to the margin of the sample at position $I[t]$.\\ 
\STATE Let $\mathbf{U}^{(k)}=\{\obs\in\UnLblTrn\big| |\langle\w^{(k)},\obs\rangle| \geq \gamma^{(k)}\}$.
\IF{$|\mathbf{U}^{(k)}| > 0$} 
\STATE $\mathbf{S}_u^{(k)}=\{(\obs,y)\big| \obs\in \mathbf{U}^{(k)}\wedge y=\text{sign}(\langle\w^{(k)},\obs\rangle)\}$
\STATE $\mathbf{S}^{(k+1)} \leftarrow \mathbf{S}^{(k)} \cup \mathbf{S}_u^{(k)}$
\STATE $\UnLblTrn \leftarrow \UnLblTrn \setminus \mathbf{U}^{(k)}$
\ELSE 
\STATE $L = L \cup [(\mathbf{w}^{(k)},\gamma^{(k)})]$
\STATE $\mathbf{S}^{(k+1)}=\{(\obs,y)\in \S^{(k)}\big| |\langle\w^{(k)},\obs\rangle| < \gamma^{(k)}\}$\\
\ENDIF
\STATE Set $k \leftarrow k + 1, w = \frac{|\mathbf{S}^{(k)}|}{p}$\\
\ENDWHILE
\STATE {\bfseries Output :} $L_m = [(\mathbf{w}^{(1)},\gamma^{(1)}),...,(\mathbf{w}^{(m)},\gamma^{(m)})]$
\end{algorithmic}
\end{tcolorbox}
\end{center}
\label{alg}\end{algorithm}

The goal of the \textit{exploration} phase is to discover the halfspace with the highest margin on the set of unlabeled samples that are not still pseudo-labeled. This is done by first, learning a halfspace that minimizes the empirical surrogate loss of $\mathcal{R}_\Dist(\w)=\Expect_{(\obs,y) \sim \Dist}[\ell_p(y,h_\w(\obs))]$ over a set of labeled and already pseudo-labeled examples $\S^{(k)}$ from $\LblTrn$ and $\UnLblTrn$:
\begin{align}
\label{eq_emp_risk}
\small\displaystyle \min_{\w}\mathcal{\hat R}_{\S^{(k)}}(\w)&= \frac{1}{|\S^{(k)}|} \sum_{(\obs,y)\in \S^{(k)}} \ell_p(y,h_\w(\obs))\\
\textrm{s.t.} & ~~||\w||_2 \leq 1 \nonumber
\end{align}
At round $k=0$, we have $\S^{(0)}=\LblTrn$. Once the halfspace with parameters 
$\w^{(k)}$ is found, a threshold $\gamma^{(k)}$, defined as the highest unsigned-margin in $\S^{(k)}$, is set such that the empirical loss over the set of examples in $\S^{(k)}$ with unsigned-margin above $\gamma^{(k)}$, is the lowest. In the pseudo-code of the algorithm, $\mathbf{S}_{\geq i}^{(k)}$ refers to the subset of examples in $\S^{(k)}$ having an unsigned margin greater or equal to $\omega\times i$.  Unlabeled examples $\obs\in\UnLblTrn$ that are not pseudo-labeled are assigned labels, i.e., $y=\text{sign}(\langle\w^{(k)},\obs\rangle)$  iff $|\langle\w^{(k)},\obs\rangle|\geq \gamma^{(k)}$. These pseudo-labeled examples are added to $\S^{(k)}$ and removed from $\UnLblTrn$, and a new halfspace minimizing Eq.~\eqref{eq_emp_risk} is found. Examples in $\S^{(k)}$ are supposed to be misclassified  by the oracle $\mathcal{O}(f, \mathcal{D}_\mathbf{x}, \eta^{(k)})$  following Definition~\ref{def} with the parameter function $\eta^{(k)}$ that refers to the conditional probability of corruption in $\mathbf{S}^{(k)}$ defined as $\eta^{(k)}(\mathbf{x}) = \Prob_{y \sim \S^{(k)}_y(\obs)}{[f(\obs) \neq y]} \leq \boldsymbol{\eta}^{(k)}$. 

Once the halfspace with parameters $\w^{(k)}$ and threshold $\gamma^{(k)}$ are found such that there are no more unlabeled samples having an unsigned-margin larger than $\gamma^{(k)}$, the pair $(\w^{(k)},\gamma^{(k)})$ is added to the list $L_m$, and samples from $\S^{(k)}$ having an unsigned-margin above $\gamma^{(k)}$ are removed (pruning phase). Remind that $\gamma^{(k)}$ is the largest threshold above which the misclassification error over $\S^{(k)}$ increases.

To classify an unknown example $\mathbf{x}$, the prediction of the first halfspace with normal vector $\w^{(i)}$ in  the list $L_m$, such that the unsigned-margin $|\langle \mathbf{w}^{(i)}, \mathbf{x} \rangle|$ of $\obs$ is higher or equal to the corresponding threshold $\gamma^{(i)}$, is returned. By abuse of notation, we note that the prediction for $\obs$ is $L_m(\obs)=h_{\w^{(i)}}(\obs)$. From Claim \ref{claim_main}, we know that the misclassification error of this halfspace on the region where $\obs$ lies is bounded by the optimal misclassification error $\mathbf{\eta^*}$. If no such halfspace exists, the observation is classified using the prediction of the first classifier $h_{\w^{(1)}}$ that was trained over all the labeled and the pseudo-labeled samples without pruning; i.e., $L_m(\obs)=h_{\w^{(1)}}(\obs)$.

\section{Corruption noise modeling and Generalization guarantees}
\label{sec4}
In the following, we relate the process of pseudo-labeling to the corruption noise model $\mathcal{O}(f, \mathcal{D}_\mathbf{x}, \eta^{(k)})$ for all pseudo-labeling iterations $k$ in Algorithm~1, then we present a bound over the misclassification error of the classifier $L_m$ outputted by the algorithm and demonstrate that this misclassification error is upper-bounded by the misclassification error of the fully supervised halfspace. 

\begin{claim} \label{claim_noise} Let $\mathbf{S}^{(0)}=\LblTrn$ be a labeled set drawn i.i.d. from $\mathcal{D}=\mathcal{O}(f, \mathcal{D}_\mathbf{x}, \eta^{(0)})$ and $\mathbf{U}^{(0)}=\UnLblTrn$ an initial unlabeled set drawn i.i.d. from $\mathcal{D}_\mathbf{x}$. For all iterations $k \in [K]$ of Algorithm~1; the active labeled set $\mathbf{S}^{(k)}$ is drawn i.i.d. from $\mathcal{D}=\mathcal{O}(f, \mathcal{D}_\mathbf{x}, \eta^{(k)})$ where the corruption noise distribution $\eta^{(k)}$ is bounded by:
\begin{equation*}
    \forall k \in [K], \quad \Expect_{\mathbf{x} \sim \mathcal{D}_\mathbf{x}}{[\eta^{(k)}(\mathbf{x}) \big| \mathbf{x} \in \mathbf{S}^{(k)}]} \leq \underset{j \in [K]}{\max} \boldsymbol{\eta}^{(j)}
\end{equation*}
\end{claim}

\textit{Proof.} We know that $\forall k \in [K], \mathbf{S}^{(k)} \subseteq \mathbf{S}^{(0)} \cup \bigcup_{i=0}^{k-1} \mathbf{S}_u^{(i)}$, where $\mathbf{S}_u^{(i)}$ is the set of pseudo-labeled pairs of examples $\mathbf{x}$ from $\mathbf{U}^{(i)}$, $\mathbf{S}_u^{(i)} = \emptyset$ for the iterations $i \in [K]$ when no examples are pseudo-labeled. Then the noise distribution $\eta^{(k)}$ satisfies for all $k \in [K]$:
\begin{align*}
\Expect_{\obs \sim \mathcal{D}_\mathbf{x}}{[\eta^{(k)}(\obs) \mathbbm{1}_{\mathbf{x} \in \mathbf{S}^{(k)}}]}  &= \Expect_{\obs \sim \mathcal{D}_\mathbf{x}}{[\eta^{(k)}(\obs) \mathbbm{1}_{\mathbf{x} \in \mathbf{S}^{(k)} \cap \mathbf{S}^{(0)}}]} + \\
&\sum_{i=0}^{k-1} \Expect_{\obs \sim \mathcal{D}_\mathbf{x}}{[\eta^{(k)}(\obs) \mathbbm{1}_{\mathbf{x} \in \mathbf{S}^{(k)} \cap \mathbf{S}_u^{(i)}}]} 
\end{align*}
If we condition on $\mathbf{x} \in \mathbf{S}^{(k)}$, we obtain for all $k \in [K]$:
\begin{multline*}
\Expect_{\obs \sim \mathcal{D}_\mathbf{x}}{[\eta^{(k)}(\obs) \big| \mathbf{x} \in \mathbf{S}^{(k)}]}  = 
\Prob{[\mathbf{x} \in \mathbf{S}^{(0)} \big| \mathbf{x} \in \mathbf{S}^{(k)}]} \Expect_{\obs \sim \mathcal{D}_\mathbf{x}}{[\eta^{(0)}(\mathbf{x}) \big| \mathbf{x} \in \mathbf{S}^{(k)} \cap \mathbf{S}^{(0)}]} + \\
\sum_{i=0}^{k-1}  \Prob{[\mathbf{x} \in \mathbf{S}_u^{(i)} \big| \mathbf{x} \in \mathbf{S}^{(k)}]} \Expect_{\mathbf{x} \sim \mathcal{D}_\mathbf{x}}{[\eta_{\mathbf{w}^{(i)}}(\mathbf{x})\big| \mathbf{x} \in \mathbf{S}^{(k)} \cap \mathbf{S}_u^{(i)}]},
\end{multline*}
this equation includes the initial corruption of the labeled set $\mathbf{S}^{(0)}=\mathbf{S}_\ell$ in addition to the noise injected by each classifier $h_{\mathbf{w}^{(k)}}$ at each round $k$ when pseudo-labeling occurs. Now that we have modeled the process of pseudo-labeling, the result is straightforward considering the fact that $\Expect_{\mathbf{x} \sim \mathcal{D}_\mathbf{x}}[\eta^{(0)}(\mathbf{x})] \leq \boldsymbol{\eta}^{(0)}; \forall k  \in [K], \Expect_{\mathbf{x} \sim \mathcal{D}_\mathbf{x}}[\eta_{\mathbf{w}^{(k)}}(\mathbf{x})] \leq \boldsymbol{\eta}^{(k)};$ and,
\begin{multline*}
\!\!\!\!\!\Prob_{\mathbf{x} \sim \mathcal{D}_\mathbf{x}}{[\mathbf{x} \in \mathbf{S}^{(0)} \big| \mathbf{x} \in \mathbf{S}^{(k)}]} + \sum_{i=0}^{k-1}  \Prob_{\mathbf{x} \sim \mathcal{D}_\mathbf{x}}{[\mathbf{x} \in \mathbf{S}_u^{(i)} \big| \mathbf{x} \in \mathbf{S}^{(k)}]} = \\ \Prob_{\mathbf{x} \sim \mathcal{D}_\mathbf{x}}{[\mathbf{x} \in \mathbf{S}^{(0)} \cup \bigcup_{i=0}^{k-1} \mathbf{S}_u^{(k)} \big| \mathbf{x} \in \mathbf{S}^{(k)}]}  \leq 1. ~~~~\square
\end{multline*}
We can now state our main contribution that bounds the generalization error of the classifier $L_m$ outputted by Algorithm~1 with respect to the optimal misclassification error $\eta^*$ in the case where projected SGD is used for the minimization of Eq.~\eqref{eq_emp_risk}. Note that in this case the time complexity of the algorithm is polynomial with respect to the dimension $d$, the upper bound on the bit complexity of examples,  the total number of iterations, and the upper bound on SGD steps.

\begin{theorem}
\label{theo1}
Let $\mathbf{S}_\ell$ be a set of i.i.d. samples of size $\ell$ drawn from a distribution $\mathcal{D} = \mathcal{O}(f, \mathcal{D}_\mathbf{x}, \eta^{(0)})$ on $\R^d \times \{-1,+1\}$, where $f$ is an unknown concept function and $\eta^{(0)}$ an unknown parameter function bounded by $1/2$, let $\mathbf{X}_u$ be an unlabeled set of size $u$ drawn i.i.d. from $\mathcal{D}_\mathbf{x}$. Algorithm~1 terminates after $K$ iterations, and outputs a non-proper classifier $L_m$ of $m$ halfspaces such that with high probability: 
\begin{equation*}
    \Prob_{(\mathbf{x},y) \sim \mathcal{D}}{[L_m(\mathbf{x}) \neq y]} \leq \bm \eta^* + \underset{k \in I}{\max}\;\epsilon^{(k)} + \pi_{K+1},
\end{equation*}
where $I$ is the set of rounds $k \in [K]$ at which the halfspaces were added to $L_m$, $\epsilon^{(k)}$ is the projected SGD convergence error rate at round $k$, and $\pi_{K+1}$ a negligible not-accounted mass of $\Dist_\obs$.
\end{theorem}

The proof of Theorem \ref{theo1} is based on the following property of projected SGD.

\begin{lemma}[From \cite{Duc16}] \label{lem_duc} Let $\mathcal{\hat R}$ be a convex function of any type. Consider the projected SGD iteration, which starts with $\w^{(0)}$ and computes for each step.
$\w^{(t+\frac{1}{2})} = \mathbf{w}^{(t)} - \alpha^{(t)} g^{(t)}; 
    \w^{(t+1)}  = \argmin_{\mathbf{w}:||\mathbf{w}||_2 \leq 1} || \mathbf{w} - \mathbf{w}^{(t+\frac{1}{2})} ||_2.$
Where $g^{(t)}$ is a stochastic subgradient such that $\Expect_{\obs \sim \Dist_\obs}{[g(\w, \obs)]} \in \partial \mathcal{\hat R}(\w)= \{g:\mathcal{\hat R}(\w') \geq \mathcal{\hat R}(\mathbf{w}) + \langle \Expect{[g]},\w' - \w \rangle \text{ for all } \w'\}$ and $\Expect_{\obs \sim \Dist_\obs}{[||g(\w, \obs)||_2^2] \leq M^2}$. For any $\epsilon, \delta > 0$; if the projected SGD is executed $T= \Omega(\log(1/\delta)/\epsilon^2)$ times with a step size $\alpha^{(t)} = \frac{1}{M\sqrt{t}}$, then for $\Bar\w = \frac{1}{T} \sum_{t=1}^{T} \w^{(t)}$, we have with probability at least $1-\delta$ that $\Expect_{\obs \sim \Dist_\obs}{[\mathcal{\hat R}(\Bar\w)]} \leq \underset{\w, \|\w\|_2 \leq 1}{\min} \Expect_{\obs \sim \Dist_\obs}{[\mathcal{\hat R}(\w)]} + \epsilon$.
\end{lemma}

\textit{Proof of Theorem}~\ref{theo1}.
We consider the steps of Algorithm\,1. At iteration $k$ of the while loop, we consider the active training set $\S^{(k)}$ consisting of examples not handled in previous iterations.

We first note that the algorithm terminates after at most $K$ iterations. From the fact that at every iteration $k$, we discard a non-empty set from $\S^{(k)}$ when we do not pseudo-label or from $\mathbf{U}^{(k)}$ when we pseudo-label, and that the empirical distributions $\LblTrn$ and $\UnLblTrn$ are finite sets. By the guarantees of Lemma~\ref{lem_duc}, running SGD (step~4) on $\mathcal{\hat R}_{\mathbf{S}^{(k)}}$ for $T = \Omega(\log(1/\delta)/\epsilon^2)$ steps, we obtain a weight vector $\mathbf{w}^{(k)}$ such that with probability at least $1-\delta$:
\begin{align*}
\Expect_{\obs \sim \Dist_\obs}[\mathcal{\hat R}_{\mathbf{S}^{(k)}}(\w^{(k)})] \leq \underset{\w, \|\w\|_2 \leq 1}{\min} \Expect_{\obs \sim \Dist_\obs}[\mathcal{\hat R}_{\mathbf{S}^{(k)}}(\w)] + \epsilon^{(k)},
\end{align*}
from Claim~\ref{claim_err_loss}, we derive with high probability:
\begin{align*}
    \Expect_{\obs \sim \Dist_\obs}[|\langle \w^{(k)}, \obs \rangle|\eta_{\w^{(k)}}(\obs)] \leq
         \underset{\w, \|\w\|_2 \leq 1}{\min} \Expect_{\obs \sim \Dist_\obs}[|\langle \w, \obs \rangle|\eta_{\w}(\obs)] + \epsilon^{(k)}.
\end{align*}
Then the margin $\gamma^{(k)}$ is estimated minimizing Eq.~\eqref{op} given $\w^{(k)}$, following Lemma~\ref{lem11} with $R^{(k)} = \displaystyle\mathop{\max}_{\obs \sim \Dist_\obs}\|\obs\|_2$ the radius of the truncated support of the marginal distribution $\Dist_\obs$ at iteration $k$, we can assume that $\frac{\gamma^{(k)}}{R^{(k)}} \approx 1, \, \forall k \in [K]$, one may argue that the assumption is unrealistic knowing that the sequence of $(\gamma^{(k)})_{k=1}^m$ decreases overall, but as we show in the supplementary, we prove in Theorem~B.1 that under some convergence guarantees of the pairs $\{(\w^{(k)}, \w^{(k+1)})\}_{k=1}^{m-1}$, one can show that the sequence $\{R^{(k)}\}_{k=1}^m$ decreases as a function of $\gamma^{(k)}$ respectively to $k$. As a result, we can derive with high probability:
\begin{align*}
    \Expect_{\obs \sim \Dist_\obs}[\eta_{\w^{(k)}}(\obs) \big| |\langle\w^{(k)}, \obs\rangle| \geq \gamma^{(k)}] \leq
         \underset{\w, \|\w\|_2 \leq 1}{\min} \Expect_{\obs \sim \Dist_\obs}[\eta_{\w}(\obs)\big| |\langle\w, \obs\rangle| \geq \gamma^{(k)}] + \epsilon^{(k)}.
\end{align*}

From the statement of Claim~\ref{claim_main} and giving the pair $(\w^{(k)}, \gamma^{(k)})$, we obtain with high probability that at round $k$:
\begin{equation}
    \label{gar}
    \Prob_{(\obs, y) \sim \Dist}{[h_{\w^{(k)}}(\obs) \neq y \big| |\langle \w^{(k)}, \obs\rangle| \geq \gamma^{(k)}]} \leq \pmb{\eta}^* + \epsilon^{(k)}.
\end{equation}

When the while loop terminates, we have accounted $m \leq K$ halfspaces in the list $L_m$ satisfying Eq.~\eqref{gar}.  For all $k \in I$, every classifier $h_{\mathbf{w}^{(k)}}$ in $L_m$ has guarantees on an empirical distribution mass of at least $ \widetilde \kappa = \underset{k \in I}{\min}\;\Prob_{\mathbf{x} \sim \mathbf{S}^{(k)}}{[|\langle \w^{(k)}, \mathbf{x} \rangle| \geq \gamma^{(k)}]}$; the DKW (Dvoretzky-Kiefer-Wolfowitz) inequality   \cite{Dvo56} implies that the true probability mass $\kappa = \underset{k \in I}{\min}\;\Prob_{\mathbf{x} \sim \mathcal{D}_\mathbf{x}}{[|\langle \w^{(k)}, \mathbf{x} \rangle| \geq \gamma^{(k)}]}$ of this region is at least $\widetilde\kappa-\sqrt{\frac{\log\frac{2}{\delta}}{2|\mathbf{S}^{(n)}|}}$ with probability $1-\delta$, where $n=\underset{k \in I}{\argmin}\;\Prob_{\mathbf{x} \sim \mathbf{S}^{(k)}}{[|\langle \w^{(k)}, \mathbf{x} \rangle| \geq \gamma^{(k)}]}$. \\The pruning phase in the algorithm ensures that these regions are disjoint for all halfspaces in $L_m$, it follows that using the Boole–Fr\'echet inequality \cite{boole15} on the conjunctions of Eq.~\eqref{gar} overall rounds $k \in [I]$, implies that $L_m$ classifies at least a $(1-m\kappa)$-fraction of the total probability mass of $\mathcal{D}$ with guarantees of Eq.~\eqref{gar} with high probability, let $\pi_{K+1}~=
~\Prob_{\mathbf{x} \sim \mathcal{D}_\mathbf{x}}{[x \in \mathbf{S}^{(K+1)}]}$ be the probability mass of the region not accounted by $L_m$. We argue that this region is negligible from the fact that $|\mathbf{S}^{(K+1)}| < \ell$ and $\ell \ll u$, such that setting $\bm \epsilon = \underset{k \in I}{\max}\;\epsilon^{(k)} + \pi_{K+1}$ provides the result.$~~~~~~~~~~~~~~~~~~~~~~\square$

In the following, we show that the misclassification error of the classifier $L_m$ output of Algorithm~1 is at most equal to the error of the supervised classifier obtained over the labeled training set  $\mathbf{S}_\ell$, when using the same learning procedure. This result suggests  that the use of unlabeled data in Algorithm~1 does not degrade the performance of the initial supervised classifier.

\begin{theorem} Let $\LblTrn$ be a set of i.i.d. samples of size $\ell$ drawn from a distribution $\mathcal{D} = \mathcal{O}(f, \Dist_\obs, \eta^{(0)})$ on $\R^d \times \{-1,+1\}$, where $f$ is an unknown concept function and $\eta^{(0)}$ an unknown parameter function bounded by $1/2$, let $\UnLblTrn$ be an unlabeled set of size $u$ drawn i.i.d. from $\Dist_\obs$. Let $L_m$ be the output of Algorithm\,1 on input $\LblTrn$ and $\UnLblTrn$, and let $h_{\mathbf{w}^{(0)}}$ be the halfspace of the first iteration obtained from the empirical distribution $\mathbf{S}^{(0)}=\LblTrn$, there is a high probability that:

\begin{equation*}
    \Prob_{(\obs,y) \sim \Dist}{[L_m(\obs) \neq y]} \leq \Prob_{(\obs,y) \sim \Dist}{[h_{\w^{(0)}}(\obs) \neq y]}
\end{equation*}
\label{theo0}
\end{theorem}

\begin{proof}
Let $k$ be the iteration at which the first pair $(\mathbf{w}^{(1)}, \gamma^{(1)})$ is added to $L_m$. The first \textit{pruning} phase in Algorithm~1 results in a set $\mathbf{S}^{(k)} \subseteq \mathbf{S}_\ell \cup \bigcup_{i=1}^{k-1} \mathbf{S}_u^{(i)}$.  Claim~\ref{claim_noise} ensures that the probability of corruption in the pseudo-labeled set $\bigcup_{i=1}^{k-1} \mathbf{S}_u^{(i)}$ is bounded by $\underset{j \in [k]}{\max}\,\boldsymbol{\eta}^{(j)} \leq \bm \eta^* + \bm \epsilon$.\\
In other words, the weight vector $\w^{(1)}$ is obtained from an empirical distribution that includes both the initial labeled set $\mathbf{S}_\ell$ and a pseudo-labeled set from $\mathbf{X}_u$. Particularly, if this pseudo-labeled set is not empty, then its pseudo-labeling error is nearly optimal, which implies that $\Prob_{(\obs,y) \sim \Dist}{[h_{\w^{(1)}}(\obs) \neq y]} \leq \Prob_{(\obs,y) \sim \Dist}{[h_{\w^{(0)}}(\obs) \neq y]}$. 

Ultimately, $L_m$ classifies a large fraction of the probability mass of $\mathcal{D}$ with nearly optimal guarantees (e.i., Eq.~\eqref{gar} in proof of Theorem~\ref{theo1}) and the rest using $h_{\w^{(1)}}$ with an error of misclassification at most equal to $\Prob_{(\obs,y) \sim \Dist}{[h_{\w^{(0)}}(\obs) \neq y]}$.
\end{proof}

\section{Empirical Results}
\label{sec5}
We compare the proposed approach to state-of-the-art strategies developed over the three fundamental working assumptions in semi-supervised learning over ten publicly available datasets. We shall now describe the corpora and methodology.
\paragraph{Datasets.} We mainly consider benchmark data sets from \cite{Chap06}. Some of these collections such as \textit{baseball-hockey}, \textit{pc-mac} and \textit{religion-atheism} are binary classification tasks extracted from the 20-newsgroups data set. 

\begin{table}[H]
    \centering
    \begin{adjustbox}{max width=\columnwidth}
    \begin{tabular}{cccccc}
        data set & $d$ & $-1$ & $+1$ & $\ell+u$ & test \\ \hline
        one-two & $64$ & $177$ & $182$ & $251$ & $108$ \\
        banknote & $4$ & $762$ & $610$ & $919$ & $453$ \\
        odd-even & $64$ & $906$ & $891$ & $1257$ & $540$\\
        pc-mac & $3868$ & $982$ & $963$ & $1361$ & $584$\\
        baseball-hockey & $5724$ & $994$ & $999$ & $1395$ & $598$\\
        religion-atheism & $7829$ & $1796$ & $628$ & $1696$ & $728$\\
        spambase & $57$ & $2788$ & $1813$ & $3082$ & $1519$ \\
        weather & $17$ & $43993$ & $12427$ & $37801$ & $18619$\\
        delicious2 & $500$ & $9610$ & $6495$ & $12920$ & $3185$\\
        mediamill2 & $120$ & $15969$ & $27938$ & $30993$ & $12914$\\
      \bottomrule
    \end{tabular}
    \end{adjustbox}
    \caption{data set statistics, $-1$ and $+1$ refer to the size of negative and positive class respectively, and test is the size of test set.}
    \label{tab:stats}
\end{table}

We used tf-idf representation for all textual data sets above. \textit{spambase} is a collection of spam e-mails from the UCI repository \cite{Dua:2019}.  \textit{one-two}, \textit{odd-even} are handwritten digits recognition tasks originally from optical recognition of handwritten digits 
database also from UCI repository, \textit{one-two} is digits "1" versus "2"; \textit{odd-even} is the artificial task of classifying odd "1, 3, 5, 7, 9" versus even "0, 2, 4, 6, 8" digits. \textit{weather} is a data set from Kaggle which contains about ten years of daily weather observations from many locations across Australia, and the objective is to classify next-day rain target variable. We have also included data sets from extreme classification repository \cite{Bha16} \textit{mediamill2} and \textit{delicious2} by selecting the label which gives the best ratio in class distribution. The statistics of these data sets are given in Table~\ref{tab:stats}.

\paragraph{Baseline methods.} We implemented the halfspace or Linear Threshold Function (LTF) using TensorFlow 2.0 in python aside with Algorithm~1\footnote{For research purposes, the code will be freely available.} ($L_m$), we ran a Support Vector Machine (SVM) \cite{vap95} with a linear kernel from the LIBLINEAR library \cite{fan08} as another supervised classifier. We compared results with a semi-supervised Gaussian naive Bayes model (GM) \cite{Chap06} from the scikit-learn library. The working hypothesis behind (GM) is the cluster assumption stipulating that data contains homogeneous labeled clusters, which can be detected using unlabeled training samples. We also compared results with  label propagation (LP) \cite{Zhu02}  which is a semi-supervised graph-based technique. We used the implementation of LP from the scikit-learn library. This approach follows the manifold assumption  that the decision boundary is located on a low-dimensional manifold and that unlabeled data may be utilized to identify it. We also included entropy regularized logistic regression (ERLR) proposed by \cite{Gra05}  from \cite{Kri17}. This approach is based on low-density separation that stipulates that the decision boundary lies on low-density regions. In the implementation of \cite{Kri17}, the initial supervised classifier is a logistic regression that has a similar performance to the SVM classifier. We tested these approaches with relatively small labeled training sets $\ell\in\{10,50,100\}$,  and because labeled information is scarce, we used the default hyper-parameters for all approaches.

\paragraph{Experimental setup.} In our experiments, we have randomly chosen $70\%$ of each data collection for training and the remaining $30\%$ for testing. We randomly selected sets of different sizes (i.e., $\ell\in\{10,50,100\}$) from the training set as labeled examples; the remaining was considered as unlabeled training samples. Results are evaluated over the test set using the accuracy measure. Each reported performance value is the average over the
$20$ random (labeled/unlabeled/test) sets of the initial collection. All experiments are carried out on a machine with an Intel Core i$7$ processor, $2.2$GhZ quad-core, and $16$Go $1600$ MHz of RAM memory.

\begin{table}[t]
    \centering
    \caption{Mean and standard deviations of accuracy on test sets over the $20$ trials for each data set. The best and the second-best performance are respectively in bold and underlined. $^\downarrow$ indicates statistically significantly worse performance than the best result,
    according to a Wilcoxon rank-sum test ($p < 0.01$) \cite{Lehman75}.}
    \begin{adjustbox}{max width=\textwidth}
    \begin{tabular}{c|c|cc|cccc}
    \hline
        Dataset & $\ell$ &  SVM & LTF & LP & GM & ERLR & $L_m$
           \\ \hline 
         \multirow{3}{*}{one-two} & $10$ & $61.38\pm13.71^\downarrow$ & $70.87\pm13.24^\downarrow$  &  $48.61\pm3.98^\downarrow$ & $\underline{75.09\pm1.30}$& $53.65\pm10.65^\downarrow$ & $\mathbf{77.77\pm1.75}$   \\
         & $50$ & $\mathbf{92.77\pm3.05}$ & $88.00\pm3.24^\downarrow$  & $49.35\pm4.20^\downarrow$ & $84.67\pm4.98^\downarrow$ & $75.78\pm8.74^\downarrow$& $\underline{91.34\pm3.21}$ \\
         & $100$ & $\mathbf{96.15\pm1.38}$ & $92.50\pm1.43^\downarrow$   & $67.82\pm12.99^\downarrow$ & $86.52\pm3.26^\downarrow$ & $79.25\pm6.87^\downarrow$ & $\underline{94.62\pm2.46}$ \\ \hline

         \multirow{3}{*}{banknote} & $10$ & $57.50\pm7.21^\downarrow$ & $\underline{69.40\pm5.53}^\downarrow$ &   $55.98\pm2.00^\downarrow$ & $69.04\pm4.60^\downarrow$& $56.71\pm4.53^\downarrow$ & $\mathbf{77.24\pm3.81}$   \\
         & $50$ & $61.67\pm4.86^\downarrow$ & $\underline{82.31\pm2.13}^\downarrow$  &  $56.28\pm1.89^\downarrow$ & $75.48\pm5.30^\downarrow$ & $65.95\pm2.01^\downarrow$ & $\mathbf{85.64\pm5.36}$ \\
         & $100$ & $71.65\pm6.24^\downarrow$ & $\underline{89.38\pm3.24}$ & $57.20\pm2.19^\downarrow$ & $77.56\pm4.34^\downarrow$  & $70.95\pm3.24^\downarrow$ & $\mathbf{90.82\pm3.31}$\\  \hline
         
         \multirow{3}{*}{odd-even} & $10$ & $53.45\pm4.80^\downarrow$ & $58.20\pm4.71^\downarrow$ & $50.37\pm1.95^\downarrow$ & $\underline{60.69\pm7.48}$& $50.40\pm2.21^\downarrow$  & $\mathbf{63.21\pm7.51}$ \\
         & $50$ & $64.75\pm5.65^\downarrow$ & $\underline{76.84\pm2.99}^\downarrow$   & $50.37\pm1.95^\downarrow$ & $62.67\pm5.82^\downarrow$& $53.17\pm4.80^\downarrow$   & $\mathbf{80.61\pm3.10}$ \\
         & $100$ & $75.89\pm6.25^\downarrow$ & $\underline{77.68\pm4.56}^\downarrow$  & $53.37\pm1.95^\downarrow$ & $64.25\pm8.18^\downarrow$ & $59.23\pm6.28^\downarrow$ & $\mathbf{84.58\pm2.12}$ \\  \hline  
         
         \multirow{3}{*}{pc-mac} & $10$ & $51.00\pm3.22^\downarrow$ & $\underline{54.92\pm2.00}^\downarrow$ & $50.93\pm1.59^\downarrow$ & $54.76\pm3.42^\downarrow$ & $50.14\pm2.06^\downarrow$ & $\mathbf{57.75\pm3.19}$ \\
         & $50$ & $58.85\pm5.09^\downarrow$ & $\underline{61.78\pm2.86}^\downarrow$ &  $50.83\pm2.08^\downarrow$ & $58.78\pm4.31^\downarrow$ & $49.71\pm1.99^\downarrow$  & $\mathbf{64.31\pm3.55}$ \\
         & $100$ & $64.57\pm4.42^\downarrow$ & $\underline{67.98\pm2.37}$ &  $50.76\pm2.26^\downarrow$ & $62.49\pm1.88^\downarrow$  & $50.36\pm2.19^\downarrow$& $\mathbf{68.15\pm5.66}$\\   \hline 
         
         \multirow{3}{*}{baseball-hockey} & $10$ & $51.57\pm2.98^\downarrow$ & $55.41\pm3.16^\downarrow$  & $\mathbf{56.53\pm5.18}$ & $49.86\pm1.77^\downarrow$  &  $49.88\pm1.89^\downarrow$& $\underline{56.47\pm5.50}$ \\
         & $50$ & $58.66\pm6.90^\downarrow$ & $\underline{69.29\pm4.32}$ &  $50.11\pm1.84^\downarrow$ & $66.76\pm5.40^\downarrow$ & $50.16\pm1.90^\downarrow$ & $\mathbf{72.85\pm6.52}$  \\
         & $100$ & $68.40\pm4.65^\downarrow$ & $\underline{76.25\pm2.41}^\downarrow$  & $49.97\pm1.82^\downarrow$ & $71.12\pm5.06^\downarrow$ & $50.35\pm1.89^\downarrow$ & $\mathbf{79.48\pm4.36}$\\  \hline
         
         \multirow{3}{*}{religion-atheism} & $10$ & $67.30\pm6.95$ & $57.30\pm4.89^\downarrow$  & $\underline{67.59\pm6.36}$ & $60.67\pm16.37^\downarrow$  & $\mathbf{71.95\pm5.03}$ & $64.25\pm7.24^\downarrow$\\
         & $50$ & $\mathbf{74.61\pm1.62}$ & $71.79\pm1.98^\downarrow$ & $67.43\pm6.05^\downarrow$ & $69.16\pm7.88$  & $\underline{74.16\pm1.88}$& $72.47\pm2.00$ \\
         & $100$ & $\mathbf{74.66\pm1.59}$ & $73.67\pm1.76$ &  $62.84\pm19.33^\downarrow$ & $70.45\pm4.39^\downarrow$  & $73.21\pm1.75$ & $\underline{73.77\pm1.82}$\\  \hline
         
        
         \multirow{3}{*}{spambase} & $10$ & $61.20\pm5.15^\downarrow$ & $57.80\pm5.29^\downarrow$  & $60.82\pm0.84^\downarrow$ & $\mathbf{74.41\pm6.64}$ & $53.38\pm11.23^\downarrow$ & $\underline{68.92\pm5.83}^\downarrow$  \\
         & $50$ & $62.59\pm9.42^\downarrow$ & $74.99\pm6.04$ & $61.15\pm0.86^\downarrow$ & $\mathbf{78.25\pm2.62}$  & $53.63\pm9.86^\downarrow$ & $\underline{76.13\pm3.08}$\\
         & $100$ & $69.43\pm10.19^\downarrow$ & $\underline{80.07\pm4.08}$  & $61.24\pm10.26^\downarrow$ & $79.08\pm2.83^\downarrow$  & $58.21\pm6.34^\downarrow$ & $\mathbf{81.93\pm2.46}$ \\  \hline
         
         
         \multirow{3}{*}{weather} & $10$ & $74.85\pm0.51$ & $68.09\pm1.73^\downarrow$  &  $\mathbf{75.49\pm0.34}$ & $75.02\pm2.79$& $40.35\pm17.29^\downarrow$ & $\underline{75.08\pm4.18}$   \\
         & $50$ & $\underline{75.79\pm0.28}$ & $75.30\pm3.85$ & $\mathbf{77.99\pm0.31}$ & $75.68\pm2.78$ & $41.55\pm27.39^\downarrow$& $75.34\pm3.80$   \\
         & $100$ & $\mathbf{77.99\pm0.25}$ & $76.27\pm3.64$  &  $\mathbf{77.99\pm0.25}$ & $74.92\pm1.92$ & $46.00\pm24.87^\downarrow$   & $\underline{77.28\pm2.99}$ \\  \hline

         \multirow{3}{*}{delicious2} & $10$ & $\underline{51.83\pm9.88}$ & $50.59\pm2.65^\downarrow$  & $\mathbf{60.02\pm0.61}$ & $49.41\pm3.83^\downarrow$ & $51.83\pm10.42^\downarrow$  & $51.08\pm1.80^\downarrow$\\
         & $50$ & $\mathbf{60.04\pm0.62}$ & $54.78\pm2.57^\downarrow$ & $\underline{60.00\pm0.59}$ & $48.35\pm1.31^\downarrow$  & $53.48\pm8.66^\downarrow$& $55.37\pm3.33^\downarrow$ \\
         & $100$ & $\underline{58.88\pm3.70}$ & $56.04\pm1.83^\downarrow$  &  $\mathbf{59.87\pm0.67}$ & $48.92\pm0.94^\downarrow$ & $54.43\pm7.27^\downarrow$ & $56.54\pm1.87^\downarrow$ \\ \hline
         
         \multirow{3}{*}{mediamill2} & $10$ & $62.54\pm2.62^\downarrow$ & $60.98\pm6.85^\downarrow$ & $36.35\pm0.15^\downarrow$ & $\underline{63.92\pm1.71}$ & $47.24\pm14.08^\downarrow$  & $\mathbf{64.31\pm3.14}$\\
         & $50$ & $63.64\pm0.15^\downarrow$ & $60.88\pm7.45^\downarrow$ &  $36.36\pm0.15^\downarrow$ & $\mathbf{65.98\pm3.32}$  & $58.58\pm11.88^\downarrow$ & $\underline{65.41\pm4.83}$ \\
         & $100$ & $63.64\pm0.15^\downarrow$ & $64.26\pm4.79$ & $36.37\pm0.15^\downarrow$ & $\underline{67.34\pm0.73}$& $63.64\pm0.16^\downarrow$  & $\mathbf{67.80\pm2.21}$ \\  \bottomrule

    \end{tabular}
  \end{adjustbox}
    \label{tab0}
\end{table}

\paragraph{Analysis of results.}
Table~\ref{tab0} summarizes the results. We used boldface (resp. underline) to indicate the highest (resp. the second-highest) performance rate, and the symbol $^\downarrow$ indicates  that performance is significantly worse than the best result, according to a Wilcoxon rank-sum test with a $p$-value threshold of $0.01$ \cite{Lehman75}. From these results, it comes out that the proposed approach ($L_m$) consistently outperforms the supervised halfspace (LTF). This finding is in line with the result of Theorem \ref{theo0}. Furthermore, compared to other techniques, $L_m$ generally performs the best or the second-best. We also notice that in some cases, LP, GM, and ERLR outperform the supervised approaches, SVM and LTF (i.e., GM on \textit{spambase} for $\ell\in\{10,50\}$), but in other cases, they are outperformed by both SVM and LTF (i.e., GM on \textit{religion-atheism}). These results suggest that unlabeled data contain useful information for classification and that existing semi-supervised techniques may use it to some extent. They also highlight that the development of semi-supervised algorithms following the given assumptions is necessary for learning with labeled and unlabeled training data but not sufficient. The importance of developing theoretically founded semi-supervised algorithms exhibiting the generalization ability of the method can provide a better understanding of the usefulness of unlabeled training data in the learning process.

\section{Conclusion}
\label{sec6}
In this study, we present a first bound over the misclassification error of a self-training algorithm that iteratively finds a list of halfspaces from partially labeled training data. Each round consists of two steps: exploration and pruning. The exploration phase's purpose is to determine the halfspace with the largest margin and assign pseudo-labels to unlabeled observations with an unsigned-margin larger than the discovered threshold. The pseudo-labeled instances are then added to the training set, and the procedure is repeated until there are no more unlabeled instances to pseudo-label. In the pruning phase, the last halfspace with the largest threshold is preserved, ensuring that there are no more unlabeled samples with an unsigned-margin greater than this threshold and pseudo-labeled samples with an unsigned-margin greater than the specified threshold are removed. Our findings are based on recent theoretical advances in robust supervised learning of polynomial algorithms for training halfspaces under large margin assumptions with a corrupted label distribution using the Massart noise model. We ultimately show that the use of unlabeled data in the proposed self-training algorithm does not degrade the performance of the initially supervised classifier. As future work, we are interested in quantifying the real gain of learning with unlabeled and labeled training data compared to a fully supervised scheme.

\bibliographystyle{apalike}
\bibliography{biblio}
\end{document}